\documentclass{article}
\pdfoutput=1

\usepackage[utf8]{inputenc}
\usepackage{amssymb, amsmath, amsthm}
\usepackage{hyperref}

\newtheorem{lemma}{Lemma}[section]

\newtheorem{theorem}[lemma]{Theorem}

\newcommand{\real}{\mathbb{R}}
\newcommand{\dualp}[1]{\left\langle #1 \right\rangle} 

\newcommand{\EE}[2]{\mathbb{E}_{#1}\left[ #2 \right]}
\newcommand{\pr}[1]{\operatorname{Pr}\left[ #1 \right]}

\newcommand{\dom}{D}
\newcommand{\wdom}{\Theta}
\newcommand{\hs}{\mathcal{H}}
\newcommand{\ptheta}{{\bar{\theta}}}
\newcommand{\res}{\kappa}
\newcommand{\wdiff}{h}
\newcommand{\lr}{\gamma} 
\newcommand{\loss}{\ell}

\DeclareMathOperator{\relu}{ReLU}
\newcommand{\activation}{\sigma}
\newcommand{\activationp}[1]{\sigma\left(#1\right)}
\newcommand{\dactivation}{\dot{\sigma}}
\newcommand{\dactivationp}[1]{\dot{\sigma}\left(#1\right)}

\newcommand{\gaussian}[1]{\mathcal{N}(#1)}
\newcommand{\px}{\bar{x}}

\newcommand{\gp}{\Sigma}
\newcommand{\ntk}{\Gamma}
\newcommand{\entk}{\hat{\Gamma}}
\newcommand{\pentk}{\bar{\hat{\Gamma}}}
\newcommand{\ppentk}{\tilde{\bar{\hat{\Gamma}}}}

\newcommand{\w}{\lambda}
\newcommand{\pW}{\bar{W}}
\newcommand{\ppW}{\tilde{W}}
\newcommand{\pf}{\bar{f}}
\newcommand{\ppf}{\tilde{f}}
\newcommand{\pptheta}{\tilde{\theta}}

\newcommand{\CN}[1]{{C^{#1}}}

\newcommand{\CHHN}[2]{{C^{0;#1,#2}}}
\newcommand{\CND}[2]{{C^{#1}(#2)}}

\newcommand{\CHHND}[3]{{C^{0;#1,#2}(#3)}}

\newcommand{\Sd}{\mathbb{S}^{d-1}}

\newcommand{\wnorm}[1]{\left\| #1 \right\|_*}
\newcommand{\ntks}{\beta}
\newcommand{\sm}{{s}}
\newcommand{\tm}{{t}}
\newcommand{\Sm}{S}

\begin{document}

\title{Approximation and Gradient Descent Training with Neural Networks}

\author{G. Welper\footnote{Department of Mathematics, University of Central Florida, Orlando, FL 32816, USA, \texttt{gerrit.welper@ucf.edu}}}
\date{}
\maketitle

\begin{abstract}

  It is well understood that neural networks with carefully hand-picked weights provide powerful function approximation and that they can be successfully trained in over-parametrized regimes. Since over-parametrization ensures zero training error, these two theories are not immediately compatible. Recent work uses the smoothness that is required for approximation results to extend a neural tangent kernel (NTK) optimization argument to an under-parametrized regime and show direct approximation bounds for networks trained by gradient flow. Since gradient flow is only an idealization of a practical method, this paper establishes analogous results for networks trained by gradient descent.
  
\end{abstract}

\smallskip
\noindent \textbf{Keywords:} deep neural networks, approximation, gradient descent, neural tangent kernel

\smallskip
\noindent \textbf{AMS subject classifications:} 41A46, 65K10, 68T07

\section{Introduction}

It is customary to split the error of supervised learning algorithms into three components: Approximation error, estimation error and optimization errors. In this paper, we consider a unified analysis of approximation and optimization errors.

The \emph{approximation error} describes how well we can approximate a target function $f$ with a neural network $f_\theta$ in the $L_2$ norm. Typical results are of the form 
\begin{equation} \label{eq:approx}
  \begin{aligned}
    \inf_\theta \|f_\theta - f\| & \le m(\theta)^{-r}, &
    f & \in K,
  \end{aligned}
\end{equation}
where $m$ describes the size of the networks (width, depth or total number of weights) and $K$ is some compact set, e.g., bounded functions in Sobolev, Besov 
\cite{
  GribonvalKutyniokNielsenEtAl2019,
  GuhringKutyniokPetersen2020,
  OpschoorPetersenSchwab2020,
  LiTangYu2019,
  Suzuki2019,
  Yarotsky2017,
  Yarotsky2018,
  YarotskyZhevnerchuk2020,
  DaubechiesDeVoreFoucartEtAl2019,
  ShenYangZhang2019,
  LuShenYangZhang2021%
}
or Barron spaces
\cite{
  Bach2017,
  KlusowskiBarron2018,
  WeinanMaWu2019,
  LiMaWu2020,
  SiegelXu2020,
  SiegelXu2020a,
  BreslerNagaraj2020%
}.
The literature shows that neural networks are competitive or even superior to classical approximation methods. See
\cite{Welper2023a}
for a more detailed literature review and 
\cite{
  Pinkus1999,
  DeVoreHaninPetrova2020,
  WeinanChaoLeiWojtowytsch2020,
  BernerGrohsKutyniokPetersen2021%
}
for surveys.
In all these results, the network weights are hand-picked and not trained, so that it remains unclear what neural networks can provably achieve, when trained by common optimization methods.

There is also a large literature on \emph{optimization} of neural networks, which currently largely relies on linearization in over-parametrized regimes, i.e. networks with significantly more parameters than training samples. A common (linearization) argument that the current paper relies on is the neural tangent kernel (NTK) 
\cite{
  JacotGabrielHongler2018,
  LiLiang2018,
  Allen-ZhuLiSong2019,
  DuZhaiPoczosSingh2019,
  DuLeeLiEtAl2019,
  AroraDuHuEtAl2019,
  SuYang2019,
  JiTelgarsky2020,
  ChenCaoZouGu2021,
  ZouCaoZhouGu2020,
  AroraDuHuEtAl2019a,
  LeeXiaoSchoenholzBahriNovakSohlDicksteinPennington2019,
  SongYang2019,
  ZouGu2019,
  KawaguchiHuang2019,
  ChizatOyallonBach2019,
  OymakSoltanolkotabi2020,
  NguyenMondelli2020,
  BaiLee2020,
  SongRamezaniKebryaPethickEftekhariCevher2021,
  LeeChoiRyuNo2022%
}.

Due to the over-parametrized regime, these optimization results achieve zero training error in discrete sample norms and are therefore not immediately compatible with the approximation literature. There are relatively few papers
\cite{
  AdcockDexter2020,
  GrohsVoigtlaender2021,
  HaoJinSiegelXu2021,
  DrewsKohler2022,
  IbragimovJentzenRiekert2022,
  JentzenRiekert2022,
  KohlerKrzyzak2022,
  HerrmannOpschoorSchwab2022,
  SiegelXu2022%
}
that consider approximation and optimization simultaneously. 

The two papers \cite{GentileWelper2022a,Welper2023a}, show approximation results of type \eqref{eq:approx} for Sobolev smooth targets $f$ and fully connected neural networks, trained with gradient flow. The first one uses shallow networks in one dimension and the second deep networks in multiple dimensions. Since gradient flow is a non-practical idealization of vanishing learning rate, the current paper shows comparable results for regular gradient descent.

\paragraph{Overview}

Section \ref{sec:main} contains the main results and Section \ref{sec:gd-convergence} a slightly abstracted version that is used in the proofs. Sections \ref{sec:proof-shallow} and \ref{sec:proof-deep} contain the proofs of the main results.

\paragraph{Notations}
Throughout the paper, $c$ denotes a generic constant that can be different in each occurrence and $a \lesssim b$, $a \gtrsim b$, $a \sim b$ denote $a \le c b$, $a \ge cb$, $a \lesssim b \lesssim a$, respectively. The constants are independent of smoothness $\sm$ and number of weights $m$, but can depend on the number of layers $L$ and input dimension $d$.

\section{Main Results}
\label{sec:main}

Throughout this section, we train weights $\theta$ in some domain $\wdom$ of networks $f_\theta$. In correspondence to typical approximation results, for the loss function, we choose the continuous $L_2$ error 
\begin{equation} \label{eq:setup:loss}
  \loss(\theta) := \frac{1}{2} \|f_\theta - f\|_{L_2(\dom)}^2
\end{equation}
on some domain $\dom$ specified below. This corresponds to an infinite sample limit (of uniformly distributed data) and places the results in an under-parametrized regime. The loss is minimized with gradient descent
\begin{equation} \label{eq:setup:gd}
  \theta^{n+1} = \theta^n - \lr \nabla_\theta \ell(\theta^n).
\end{equation}
with learning rate $\lr$ and random initialization.

\subsection{Shallow Networks in \texorpdfstring{$1d$}{1d}}

For the first result, we choose shallow networks
\begin{equation} \label{eq:setup:network}
  f_\theta(x) = \frac{1}{\sqrt{m}} \sum_{r=1}^m a_r \activation(x - b_r)
\end{equation}
in one dimension $\dom = [-1,1]$. The weights $a_r$ are initialized with random $\pm 1$ and not trained and the biases $\theta_r := b_r$ are initialized from a uniform distribution on $\dom$ and trained. Although it may seem peculiar not to optimize the $a_r$, the given setup is intended as the simplest test case for which the loss is non-convex.

To state the main result, we use a smoothness norm, to define the compact set $K$ from the introduction, which we define analogous to a $\sin$ or Fourier transform: With basis and weights (arising naturally as eigenvectors and eigenvalues of the NTK in \cite{GentileWelper2022a})
\begin{align*}
  \phi_k(x) & = \left\{ \begin{array}{ll}
    \sin\left(\omega_k x - \frac{\pi}{4} \right) & k\text{ even} \\
    \sin\left(\omega_k x + \frac{\pi}{4} \right) & k\text{ odd.}
  \end{array} \right. &
  \omega_k & = \frac{\pi}{4} + \frac{\pi}{2} k
\end{align*}
and $\sm \in \real$, we define the Hilbert spaces $\hs^\sm$ for which the norm
\begin{equation*} 
\|v\|_\sm := \left( \sum_{k=1}^\infty \omega_k^{2\sm} \dualp{\psi_k, v}^2 \right)^{1/2}
\end{equation*}
is finite. Since the $\phi_k$ are orthogonal in $L_2(\dom)$, for $\sm=0$ the norm is equivalent to $\|\cdot\|_{L_2(\dom)}$. For $\sm \ne 0$, similar to Fourier bases, the norms are equivalent to Sobolev space $H^\sm(\dom)$, up to some modified boundary conditions.

\begin{theorem} \label{th:1d:convergence}

  Assume we train the shallow network $f_\theta$, defined in \eqref{eq:setup:network}, with gradient descent \eqref{eq:setup:gd} applied to the $L_2(\dom)$ loss \eqref{eq:setup:loss}, with learning rate $\lr \lesssim h \sqrt{m}$ and 
  \begin{align*}
    h & = c_h m^{-\frac{1}{2} \frac{1}{2-\sm}}, &
    \tau & = h^{2(1-\sm)} m.
  \end{align*}
  for some $0 < \sm < 1/2$ and some constant $c_h$ that may depend on the initial error $\|f_{\theta^0} - f\|_0$. Then, with $\res^n := f_{\theta^n} - f$ and probability at least $1 - \frac{c}{h} e^{-2m \wdiff^2} - 2 \tau \left(e^\tau - \tau - 1 \right)^{-1}$, while the gradient descent error exceeds the final approximation error
  \begin{align} \label{eq:shallow:not-finished}
    \|\res^k\|_0^2 
    & \ge c_a m^{-\frac{1}{2} \frac{1-\sm}{2-\sm} \sm} \|\res^0\|_\sm^2, &
    k & < n,
  \end{align}
  we have
  \begin{align*}
    \|\res^n\|_0^2 & \le Ce^{-\lr h^{1-\sm} n} \|\res^0\|_0^2, & 
    \|\res^n\|_\sm^2 & \le C \|\res^0\|_\sm^2.
  \end{align*}
  for sufficiently large constants $c_a$, $c$ and $C$ independent of $m$, $\res^0$ and $\res^n$.
\end{theorem}

The proof is in Section \ref{sec:1d:convergence}. As long as the training has not achieved the direct approximation inequality 
\begin{equation*} 
  \|\res^n\|_0^2 
  \le m^{-\frac{1}{2} \frac{1-\sm}{2-\sm} \sm} \|\res^0\|_\sm^2,
\end{equation*}
condition \eqref{eq:shallow:not-finished} is satisfied, and the error decays exponentially. In comparison, the networks (with trained $a_r$) are piecewise linear with $m$ breakpoints, for which one would expect approximation errors
\begin{equation*} 
  \inf_{\phi\text{ p.w.lin.}} \|\phi - f\|_0^2 
  \le m^{-\sm} \|\res^0\|_\sm^2,
\end{equation*}
with a higher rate than in Theorem \ref{th:1d:convergence}. Numerical experiments in \cite{GentileWelper2022a} confirm that the rate is lower than theoretically possible, both with trained and untrained $a_r$, but better than Theorem \ref{th:1d:convergence}.

The result allows a fairly large learning rate because the $1/\sqrt{m}$ scaling of the network implies small gradients.

\subsection{Deep Networks in Multiple Dimensions}

\paragraph{Network}

For the second result, we consider fully connected networks 
\begin{equation} \label{eq:deep:network}
  \begin{aligned}
    f^1(x) & = W^0 V x, & & \\
    f^{\ell+1}(x) & = W^\ell m_\ell^{-1/2} \activationp{f^\ell(x)}, & \ell = 1, \dots, L \\ 
    f(x) & = f^{L+1}(x), & & 
  \end{aligned}
\end{equation}
of constant depth $L$ for normalized inputs on the $d$-dimensional unit sphere $\dom = \Sd$. Except for an arbitrary initial matrix $V$ with orthonormal columns, all weights are initialized randomly and only the second but last layer weights $W^{L-1}$ are trained:
\begin{align*}
  & V \in \real^{m_0 \times d} & & \text{orthogonal columns }V^T V = I & & \text{not trained}, \\
  & W^{\ell} \in \real^{m_{\ell+1} \times m_\ell}, \, \ell=1, \dots, L-2 & & \text{i.i.d. }\gaussian{0, 1} & & \text{not trained}, \\
  & W^{L-1} \in \real^{m_L \times m_{L-1}}, & & \text{i.i.d. }\gaussian{0, 1} & & \text{trained}, \\
  & W^{L+1} \in \{-1,+1\}^{1 \times m_{L+1}} & & \text{i.i.d. Rademacher} & & \text{not trained}.
\end{align*}
As for the shallow case, this provides a non-convex optimization problem. The output is scalar and all hidden layers are of comparable size
\begin{align*}
  m & := m_{L-1}, &
  1 & = m_{L+1} \le m_L \sim \dots \sim m_0 \ge d.
\end{align*}

\paragraph{Activation Functions}

We use activation functions with no more than linear growth, uniformly bounded first and second derivatives and no more than polynomial growth of the third and fourth derivatives
\begin{align} \label{eq:deep:activation}
  |\activationp{x}| & \lesssim |x|, &
  |\activation^{(i)}(x)| & \lesssim 1 &
  i & = 1,2, &
  |\activation^{(j)}(x)| & \le p(x), &
  j & = 3,4,
\end{align}
for some polynomial $p(x)$.

\paragraph{Smoothness}

The target function $f$ is contained in Sobolev spaces $H^\sm(\Sd)$ on the sphere $\dom = \Sd$, with norms and scalar products denoted by $\|\cdot\|_{H^\sm(\Sd)}$ and $\dualp{\cdot, \cdot}_{H^\sm(\Sd)}$, see \cite{Welper2023a} for details.

\paragraph{Neural Tangent Kernel}

Unlike the shallow case, we need one more assumption on the (NTK) linearization of the networks that is currently known for non-smooth $\relu$ activations and only tested numerically for the smooth activations of our network \cite{Welper2023a}. For our result, only the second but last weights $W^{L-1}$ are trained, whereas all other weights are randomly initialized and unchanged. Therefore, in our case the NTK is defined by
\begin{equation} \label{eq:deep:ntk-limit}
  \ntk(x,y) 
  = \lim_{\text{width}\to\infty} \sum_{|\w| = L-1} \partial_\w f_r^{L+1}(x) \partial_\w f_r^{L+1}(y).
\end{equation}
with partial derivatives abbreviated by $\lambda = W^{L-1}_{ij}$ on layer $|\lambda| := L-1$. We assume that this integral kernel is coercive in Sobolev norms
\begin{align} \label{eq:deep:coercivity}
  \dualp{f, H f}_{H^{\Sm}(\Sd)}
  & \gtrsim \|f\|_{H^{\Sm - \ntks}(\Sd)}, &
  H f & := \int_\dom \ntk(\cdot,y) f(y) \, dy
\end{align}
for some $0 \le \sm \le \frac{\ntks}{2}$, all $\Sm \in \{0,\sm\}$ and all $f \in H^{\sm}(\Sd)$. This follows easily from \cite{BiettiMairal2019,GeifmanYadavKastenGalunJacobsRonen2020,ChenXu2021} for $\relu$ activations and sum over all partial derivatives (not only $|\lambda| = L-1$), but our theory requires smoother activations for which this property is only tested numerically in \cite{Welper2023a}.

Finally, we need one technical assumption that for the Gaussian process
\begin{align*}
  \gp^{\ell+1}(x,y) & := \EE{u,v \sim \gaussian{0, A}}{\activationp{u}, \activationp{v}}, &
  A & = \begin{bmatrix}
    \gp^\ell(x,x) & \gp^\ell(x,y) \\
    \gp^\ell(y,x) & \gp^\ell(y,y)
  \end{bmatrix}, & 
  \gp^0(x,y) & = x^T y,
\end{align*}
we have
\begin{align} \label{eq:deep:gp-bounds}
  c_\gp \le \gp^k(x,x) & \le C_\gp > 0, &
\end{align}
for all $x,y \in \dom$, $k = 1, \dots, L$ and constants $c_\gp, C_\gp \ge 0$. This process describes the forward evaluation of the network for initial random weights and is used in recursive NTK formulas \cite{JacotGabrielHongler2018}. It is known that the process is zonal, i.e. it only depends on the angle $x^T y$ so that $\gp^{\ell}(x,x) = \gp^\ell(x^T x) = \gp^\ell(1)$, which must be non-zero to satisfy our assumption. Again, this property is known for $\relu$ activations \cite{ChenXu2021} and expected to be simple to verify for our smoother activations. It is left to a more thourough study of the NTK that is required for coercivity.

\paragraph{Result}

\begin{theorem} \label{th:deep:convergence}

  Assume that the neural network \eqref{eq:deep:network} - \eqref{eq:deep:activation} is trained by gradient descent \eqref{eq:setup:gd} applied to the $L_2(\dom)$ loss \eqref{eq:setup:loss}. Assume:
  \begin{enumerate}
    \item The NTK satisfies coercivity \eqref{eq:deep:coercivity} for $0 \le 2 \sm \le \ntks$ and the forward process satisfies \eqref{eq:deep:gp-bounds}.
    \item All hidden layers are of similar size: $m_0 \sim \dots \sim m_{L-1} =: m$.
    \item Smoothness is bounded by $0 < \sm < 1/2$.
    \item Define $h$ and $\tau$ as follows and choose learning rate $\lr$ and an arbitrary $\alpha$ so that
    \begin{align*}
      h & = c_h m^{-\frac{1}{2} \frac{1}{1+\alpha}}, &
      \tau & = h^{2\alpha} m, &
      \lr & \lesssim h \sqrt{m}, &
      0 \le \alpha & < 1-s.
    \end{align*}
    for some constant $c_h$ that may depend on the initial error $\|f_{\theta^0} - f\|_0$.
  \end{enumerate}
  Then with $\res^n := f_{\theta^n} - f$ and probability at least $1 - c L (e^{-m} + e^{-\tau})$, while the gradient descent error exceeds the final approximation error
  \begin{align} \label{eq:deep:not-finished}
    \|\res^k\|_0^2 
    & \ge c_a m^{-\frac{1}{2} \frac{\alpha}{1+\alpha} \frac{\sm}{\ntks}} \|\res^0\|_\sm^2, &
    k & < n,
  \end{align}
  we have
  \begin{align*}
    \|\res^n\|_0^2 & \le C e^{-\lr h^\alpha n} \|\res^0\|_0^2, & 
    \|\res^n\|_\sm^2 & \le C \|\res^0\|_\sm^2.
  \end{align*}
  for sufficiently large constants $c_a$, $c$ and $C$ independent of $m$, $\res^0$ and $\res^n$.
  
\end{theorem}

The proof is in Section \ref{sec:deep:convergence}. The conclusion of the theorem is analogous to the shallow case: As long as the approximation error in \eqref{eq:deep:not-finished} is not achieved, gradient descent reduces the error exponentially. All assumptions are easy to verify, except coercivity, which is known for $\relu$ activations and tested numerically for the required smoother activations \cite{Welper2023a}.

\section{Gradient Descent Convergence}
\label{sec:gd-convergence}

Both Theorems \ref{th:1d:convergence} and \ref{th:deep:convergence} are shown by an abstracted gradient descent convergence result in this section, based on the NTK. To this end, let $\hs^\sm$ be a hierarchy of Hilbert spaces with norms $\|\cdot\|_\sm = \|\cdot\|_{\hs^\sm}$ that satisfy an interpolation inequality 
\begin{equation*} 
  \|\cdot\|_b \lesssim \|\cdot\|_a^{\frac{c-b}{c-a}} \|\cdot\|_c^{\frac{b-a}{c-a}}
\end{equation*}
for $a,b,c \in \real$ and $f_\cdot: \wdom \to \hs^0$ be a function that we train by gradient descent \eqref{eq:setup:gd} with loss $\ell(\theta) = \frac{1}{2} \|f_\theta - f\|_0^2$ for some function $f \in \hs^0$. In Theorems \ref{th:1d:convergence} and \ref{th:deep:convergence}, we use Sobolev spaces $\hs^\sm = H^\sm(\dom)$ for various domains and define $f_\theta = f_\theta(\cdot) \in L_2(\dom) = \hs^0$. 

The argument is based on linearization or the \emph{neural tangent kernel (NTK)}. With $\partial_r := \partial_{\theta_r}$, $\hs^0$ dual $(\cdot)^*$ and
\begin{equation*}
  H_{\theta, \ptheta} := \sum_r (\partial_r f_\theta) (\partial_r f_{\ptheta})^*
\end{equation*}
the NTK $H$ is the infinite width limit for initial weights $\theta = \ptheta = \theta^0$. We omit a rigorous definition, because we only need a limiting operator $H$ with the properties stated in the following Theorem.

\begin{theorem} \label{th:general:convergence}

  Assume we train the parametrized function $\theta \to f_\theta \in \hs^0$, with gradient descent \eqref{eq:setup:gd} applied to the loss $\frac{1}{2} \|f_\theta - f\|_{\hs^0}$ for some $f \in \hs^0$. Let $m$ be an indicator for the network size that satisfies the inequalities below. Assume there is some $\alpha > 0$ such that
  \begin{enumerate}

    \item \label{item:abstract:coercive} $H$ is coercive for $S=0$ and $S=\sm$ and some $\ntks > s > 0$
    \begin{align} \label{eq:abstract:coercive}
      \|v\|_{S-\ntks}^2 & \lesssim \dualp{v, H v}_S, &
      v & \in \hs^{S-\ntks}.
    \end{align}

    \item For some norm $\wnorm{\cdot}$, the distance of the weights from their initial value is bounded by
    \begin{align} \label{eq:abstract:weight-diff}
      \wnorm{\theta^k - \theta^0} & \lesssim 1, \, k=1, \dots, n-2 &
      & \Rightarrow & 
      \wnorm{\theta^{n-1} - \theta^0}
      & \lesssim \frac{\lr}{\sqrt{m}} \sum_{k=0}^{n-1} \|\res^k\|_0.
    \end{align}

    \item The learning rate $\gamma$ is sufficiently small so that
    \begin{equation} \label{eq:abstract:lr}
      \lr \wnorm{\nabla_\theta \loss(\theta^{n-1})}
      \lesssim c_h m^{-\frac{1}{2} \frac{1}{1+\alpha}}
      =: h.
    \end{equation}
    for some constant $c_h$ that may depend on the initial error $\|\res^0\|_0$.

  \item For $S=0$ and $S=\sm$, initial value $\theta^0$, any $\ptheta, \pptheta \in \wdom$ and any $\bar{h} > 0$, the bounds $\wnorm{\theta^0 - \ptheta} \le \bar{h}$ and $\wnorm{\theta^0 - \pptheta} \le \bar{h}$ imply
    \begin{align} \label{eq:assume-ntk-perturbations}
      \|H_{\pptheta, \theta^0}  - H_{\pptheta, \ptheta}\|_{S,0} & \le c \bar{h}^\alpha, & 
      \|H_{\theta^0, \pptheta}  - H_{\ptheta, \pptheta}\|_{S,0} & \le c \bar{h}^\alpha.
    \end{align}

  \item For $S=0$ and $S=\sm$, we have
    \begin{equation} \label{eq:assume-initial-close-to-ntk}
      \|H - H_{\theta^0, \theta^0}\|_{S,0} 
      \le c_h m^{-\frac{1}{2} \frac{\alpha}{1+\alpha}}
      = h^\alpha.
    \end{equation}
  \end{enumerate}
  Then, with $\res^n := f_{\theta^n} - f$, while the gradient descent error exceeds the final approximation error
  \begin{align} \label{eq:abstract:not-finished}
    \|\res^k\|_0^2 
    & \ge c_a m^{-\frac{1}{2} \frac{\alpha}{1+\alpha} \frac{\sm}{\ntks}} \|\res^0\|_\sm^2, &
    k & < n,
  \end{align}
  we have
  \begin{align*}
    \|\res^n\|_0^2 & \le C e^{-\lr h^\alpha n} \|\res^0\|_0^2, & 
    \|\res^n\|_\sm^2 & \le C \|\res^0\|_\sm^2.
  \end{align*}
  for sufficiently large constants $c_a$, $c$ and $C$ independent of $m$, $\res^0$ and $\res^n$.

\end{theorem}

Both Theorems \ref{th:1d:convergence} and \ref{th:deep:convergence} are shown by providing the assumptions of the last theorem. The proof is given at the end of this section and based on a typical NTK argument: We will see that in each step the loss is reduced by
\begin{equation*}
  \loss(\theta^{n+1}) - \loss(\theta^n)
  = - \lr \dualp{\res, H_{\theta^n - \xi \lr \nabla_\theta \loss(\theta^n), \theta^n} \res},
\end{equation*}
which leads to convergence if we can bound the right hand side away from zero. With the given perturbation and concentration inequalities, we show that the system is almost linear and coercive
\begin{equation*}
  \loss(\theta^{n+1}) - \loss(\theta^n)
  \approx - \lr \dualp{\res, H \res} + \text{perturbations}
  \lesssim \|\res\|_{-\ntks}^2 + \text{perturbations}.
\end{equation*}
The norm $\|\cdot\|_{-\ntks}$ is too weak to prove convergence by the discrete Grönwall lemma, but utilizing interpolation inequalities and smoothness allows a similar argument.

\subsection{Gradient Descent Error Reduction}

For the convergence proof, we not only control the loss $\|\res^n\|_0$, but also the smoothness $\|\res^n\|_\sm$ and therefore extend the loss to include it
\begin{align*}
  \ell_S(\theta)
  := \frac{1}{2} \|\res\|_S^2
  := \frac{1}{2} \|f_\theta - f\|_S^2,
\end{align*}
for $S=0$ and $S=\sm$, where we drop the subscript if $\sm=0$. The following lemma shows a non-zero error decay in every gradient descent step.

\begin{lemma} \label{lemma:loss-reduction}
  Assume that \eqref{eq:assume-ntk-perturbations} and \eqref{eq:assume-initial-close-to-ntk} hold. Then
  \begin{equation*}
    \loss_S(\theta^{n+1}) - \loss_S(\theta^n)
    \le - \lr \dualp{\res, H \res}_S + 3 c \lr \big[h + \lr \wnorm{\nabla_\theta \loss(\theta^n)} \big]^\alpha \|\res\|_S \|\res\|_0.
  \end{equation*}
\end{lemma}

\begin{proof}

With the gradient descent update $\theta^{n+1} = \theta^n - \Delta^n$ with $\Delta^n := \lr \nabla_\theta \loss(\theta^n)$, by the mean value theorem we have for some $\xi \in (0, 1)$
\begin{align*}
  \loss_S(\theta^{n+1}) - \loss_S(\theta^n)
  & = \loss_S(\theta^n - \Delta^n) - \loss_S(\theta^n)
  \\
  & = - \loss_S'(\theta^n - \xi \Delta^n) \Delta^n.
\end{align*}
Breaking up the derivative into partial derivatives $\partial_r = \partial_{\theta_r}$ and using that $\partial_r \loss_S(\theta) = \dualp{\res, \partial_r f_\theta }_{S}$ and the definition of $\Delta^n$, we obtain
\begin{align*}
  \loss_S(\theta^{n+1}) - \loss_S(\theta^n)
  & = - \lr \sum_r \dualp{\res, \partial_r f_{\theta^n - \xi \Delta^n} }_S \dualp{\res, \partial_r f_{\theta^n} } 
  \\
  & = - \lr \dualp{\res, \left[\sum_r (\partial_r f_{\theta^n - \xi \Delta^n}) (\partial_r f_{\theta^n})^* \right] \res }_S,
  \\
  & = - \lr \dualp{\res, H_{\theta^n - \xi \Delta^n, \theta^n} \res}_S
\end{align*}
where in the last step we have used the $\hs^0$ dual $v^* \res := \dualp{v, \res}$. Next, we add and subtract terms to compare $f_{\theta^n - \xi \Delta^n}$ and $f_{\theta^n}$ with the initial $f_{\theta^0}$ to obtain
\begin{align*}
  \loss_S(\theta^{n+1}) - \loss_S(\theta^n)
  & = - \lr \dualp{\res, H_{\theta^0, \theta^0} \res}_S
  \\
  & \quad + \lr \dualp{\res, H_{\theta^0, \theta^0}  - H_{\theta^0, \theta^n} \res}_S
  \\
  & \quad + \lr \dualp{\res, H_{\theta^0, \theta^n}  - H_{\theta^n - \xi \Delta^n, \theta^n} \res}_S.
\end{align*}
From assumption \eqref{eq:assume-ntk-perturbations}, with $\bar{h} = h$ and $\bar{h} = h + \wnorm{\Delta^n}$, respectively, we obtain
\begin{align*}
  \|H_{\theta^0, \theta^0}  - H_{\theta^0, \theta^n}\|_{S,0} & \le c h^\alpha, \\
  \|H_{\theta^0, \theta^n}  - H_{\theta^n - \xi \Delta^n, \theta^n}\|_{S,0} & \le c \left(h + \wnorm{\Delta^n} \right)^\alpha.
\end{align*}
Moreover, from assumption \eqref{eq:assume-initial-close-to-ntk} we have
\begin{equation*}
  - \dualp{\res, H_{\theta^0, \theta^0} \res}_S
  = - \dualp{\res, H \res}_S + \dualp{\res, H - H_{\theta^0, \theta^0} \res}_S
  \le - \dualp{\res, H \res}_S + h^\alpha \|\res\|_S \|\res\|_0
\end{equation*}
Combining the above inequalities, we arrive at
\begin{equation*}
  \loss_S(\theta^{n+1}) - \loss_S(\theta^n)
  \le - \lr \dualp{\res, H \res}_S + 3 c \lr \big[h + \wnorm{\Delta^n} \big]^\alpha \|\res\|_S \|\res\|_0,
\end{equation*}
which proves the lemma.
  
\end{proof}

\subsection{Auxiliary Results}

The following lemma contains a Grönwall type inequality to show convergence.

\begin{lemma} \label{lemma:sequence-bounds}

  Let $a,b,c,d > 0$ and $\rho > 1/2$. Let $x_n$ and $y_n$ be two sequences that satisfy
  \begin{equation}
    \begin{aligned} \label{eq:convergence:sequence-bounds}
      x_{n+1} - x_n & \le - \lr a x_n^{1+\rho} y_n^{-\rho} + \lr b x_n, \\
      y_{n+1} - y_n & \le - \lr c x_n^\rho y_n^{1-\rho} + \lr d \sqrt{x_n y_n}.
    \end{aligned}
  \end{equation}
  Furthermore, assume that
  \begin{align} \label{eq:convergence:condition}
    x_k & \ge \left(\frac{d}{c}\right)^{\frac{2}{2\rho -1}} y_0, &
    x_k & \ge \left(2 \frac{b}{a}\right)^{\frac{1}{\rho}} y_0, &
    \text{for all }k & = 0, \dots, n-1.
  \end{align}
  Then
  \begin{align*}
    x_n & \le e^{-\lr b n} x_0, &
    y_n & \le y_0.
  \end{align*}
  
\end{lemma}

\begin{proof}

We first show that $y_{n+1} \le y_0$. By induction, assume this to be true for $y_n$. Then, with $\rho > 1/2$, the assumptions imply
\begin{align*}
  x_n & \ge \left(\frac{d}{c}\right)^{\frac{2}{2\rho -1}} y_0, &
  & \Rightarrow & 
  x_n & \ge \left(\frac{d}{c}\right)^{\frac{2}{2\rho -1}} y_n, &
  & \Leftrightarrow & 
  - \lr c x_n^\rho y_n^{1-\rho} + \lr d \sqrt{x_n y_n} & \le 0.
\end{align*}
Hence the bounds for $y_{n+1} - y_n$ in \eqref{eq:convergence:sequence-bounds} imply that $y_{n+1} \le y_n \le y_0$, which shows the first part of the lemma. 

Next, we estimate $x_{n+1}$ by induction. From the assumptions, we have
\begin{align*}
  x_k & \ge \left(2 \frac{b}{a}\right)^{\frac{1}{\rho}} y_0 &
  & \Leftrightarrow &
  a x_n^\rho y_0^{-\rho} \ge 2 b.
\end{align*}
Thus, from the sequence bounds \eqref{eq:convergence:sequence-bounds} and $y_n \le y_0$ we conclude that
\begin{align*}
  x_{n+1} - x_n & \le - \lr a x_n^{1+\rho} y_0^{-\rho} + \lr b x_n \\
  \Leftrightarrow x_{n+1} & \le \left( 1 - \lr a x_n^\rho y_0^{-\rho} + \lr b \right) x_n \\
  & \le \left( 1 - \lr b \right) x_n \\
  & \le e^{- \lr b} x_n \\
  & \le e^{- \lr b} e^{- \lr b n} x_0 \\
  & = e^{- \lr b (n+1)} x_0,
\end{align*}
where in the third but last step we have used $1+x \le e^x$ and in the second but last step the induction hypothesis.
  
\end{proof}

\subsection{Proof of Theorem \ref{th:general:convergence}}
\label{sec:abstract:convergence}

\begin{proof}[Proof of Theorem \ref{th:general:convergence}]

  We prove the result with Lemma \ref{lemma:loss-reduction} for which we have to control the weight distance $\wnorm{\theta^n - \theta^0}$ throughout the gradient descent iteration. Assume by induction that 
\begin{align*}
  \|\res^k\|_0^2 & \lesssim e^{-\lr h^\alpha k} \|\res^k\|_0^2 \\
h^k := \max_{l \le k} \wnorm{\theta^l - \theta^0} & \lesssim c_h m^{-\frac{1}{2} \frac{1}{1+\alpha}} =: h
\end{align*}
for all $k < n$. We prove the bounds for $k=n$. With assumptions \eqref{eq:assume-ntk-perturbations}, \eqref{eq:assume-initial-close-to-ntk}, we apply Lemma \ref{lemma:loss-reduction}and combined with coercivity \eqref{item:abstract:coercive} we obtain
\begin{align*}
  \|\res^n\|_0^2 - \|\res^0\|_0^2
  & \le - \lr \|\res^{n-1}\|_{-\ntks}^2 + 3 c \lr \big[h + \lr \wnorm{\nabla_\theta \loss(\theta^n)} \big]^\alpha \|\res^{n-1}\|_0^2.
  \\
  \|\res^n\|_\sm^2 - \|\res^0\|_\sm^2
  & \le - \lr \|\res^{n-1}\|_{\sm-\ntks}^2 + 3 c \lr \big[h + \lr \wnorm{\nabla_\theta \loss(\theta^n)} \big]^\alpha \|\res^{n-1}\|_\sm \|\res^{n-1}\|_0.
\end{align*}
In order to eliminate the $\|\cdot\|_{-\ntks}$ and $\|\cdot\|_{\sm-\ntks}$ norms, we use the interpolation inequalities
\begin{align*}
  \|\res\|_0 & \le \|\res\|_{-\ntks}^{\frac{\sm}{\ntks+\sm}} \|\res\|_\sm^{\frac{\beta}{\beta+\sm}} & 
  & \Rightarrow &
  \|\res\|_{-\ntks}^2 & \le \|\res\|_0^{2+ \frac{2\ntks}{\sm}} \|\res\|_\sm^{-\frac{2\ntks}{\sm}},
  \\
  \|\res\|_0 & \le \|\res\|_{\sm-\ntks}^{\frac{\sm}{\ntks}} \|\res\|_\sm^{\frac{\beta-\sm}{\ntks}} &
  & \Rightarrow &
  \|\res\|_{\sm-\ntks}^2 & \le \|\res\|_0^{\frac{2\ntks}{\sm}} \|\res\|_\sm^{2-\frac{2\ntks}{\sm}}.
\end{align*}
Together with the learning rate bound $\lr \wnorm{\nabla_\theta \loss(\theta^n)} \lesssim h$ from assumption \eqref{eq:abstract:lr}, we arrive at
\begin{align*}
  \|\res^n\|_0^2 - \|\res^0\|_0^2
  & \lesssim - \lr \|\res^{n-1}\|_0^{2+ \frac{2\ntks}{\sm}} \|\res^{n-1}\|_\sm^{-\frac{2\ntks}{\sm}} + \lr h^\alpha \|\res^{n-1}\|_0^2,
  \\
  \|\res^n\|_\sm^2 - \|\res^0\|_\sm^2
  & \lesssim - \lr \|\res^{n-1}\|_0^{\frac{2\ntks}{\sm}} \|\res^{n-1}\|_\sm^{2-\frac{2\ntks}{\sm}} + \lr h^\alpha \|\res^{n-1}\|_\sm \|\res^{n-1}\|_0.
\end{align*}
We now estimate $x_n := \|\res\|_0^2$ and $y_n := \|\res\|_\sm^2$ by Lemma \ref{lemma:sequence-bounds} with $\rho = \ntks/\sm$, $a=c=1$ and $b = d = h^\alpha$. To verify the lemma's assumption \eqref{eq:convergence:condition}, note that by
\begin{align*}
  \left(2-\frac{\sm}{\ntks}\right) & \le 2 &
  & \Leftrightarrow &
  \frac{\sm}{\ntks} & \le \frac{2\frac{\sm}{\ntks}}{2-\frac{\sm}{\ntks}} &
  & \Leftrightarrow &
  \frac{1}{\rho} & \le \frac{2}{2\rho - 1}
\end{align*}
so that together with assumption \eqref{eq:abstract:not-finished} we have
\begin{equation*}
  x^k = \|\res^k\|_0^2 
  \ge \left(m^{-\frac{1}{2} \frac{1}{1+\alpha}} \right)^{\alpha \frac{\sm}{\ntks}} \|\res^0\|_\sm^2  
  = h^{\alpha \frac{\sm}{\ntks}} \|\res^0\|_\sm^2  
  \gtrsim \left(2 \frac{b}{a} \right)^{\frac{1}{\rho}} y_0
  \gtrsim \left(\frac{d}{c} \right)^{\frac{2}{2\rho-1}} y_0.
\end{equation*}
Hence, Lemma \ref{lemma:sequence-bounds} implies 
\begin{align*}
  \|\res^n\|_0^2 & \lesssim e^{-\lr h^\alpha n} \|\res^0\|_0^2, & 
  \|\res^n\|_\sm^2 & \lesssim \|\res^0\|_\sm^2,
\end{align*}
which shows the first induction hypothesis. It remains to show that the weights stay close to their initial value
\begin{equation*}
  h^n 
  = \max_{k \le n} \wnorm{\theta^n - \theta^0} 
  \lesssim \frac{\lr}{\sqrt{m}} \sum_{k=1}^{n-1} \|\res^k\|_0
  \lesssim \frac{\lr}{\sqrt{m}} \sum_{k=1}^{n-1} e^{- \lr h^\alpha k} \|\res^0\|_0,
\end{equation*}
where in the second step we have used assumption \eqref{eq:abstract:weight-diff} and in the third step the induction hypothesis. Computing the geometric sum
\begin{equation*}
  \sum_{k=1}^{n-1} e^{-\lr h^\alpha k}
  \le \int_0^\infty e^{-\lr h^\alpha k} \, dk
  = \frac{1}{\lr h^\alpha},
\end{equation*}
we arrive at
\begin{equation*}
  h^n 
  \le c \frac{\lr}{\sqrt{m}} \frac{1}{\lr h^\alpha} \|\res^0\|_0
  = h
\end{equation*}
where we have used that by our choice of $h$ we have
\begin{align*}
  h & = c_h m^{-\frac{1}{2} \frac{1}{1+\alpha}} &
  & \Leftrightarrow &
  h & = \frac{4}{\sqrt{m}} m^{\frac{1}{2} \frac{\alpha}{1+\alpha}} \|\res^0\|_0
      = \frac{4}{\sqrt{m}} h^{-\alpha} \|\res^0\|_0
\end{align*}
for a suitable choice of $c_h$ dependent on $\|\res^0\|_0$. This shows the second induction hypothesis and concludes the proof.

\end{proof}

\section{Proof of Main Results: Shallow \texorpdfstring{$1d$}{1d}}
\label{sec:proof-shallow}

In this section, we proof Theorem \ref{th:1d:convergence} as a special case of Theorem \ref{th:general:convergence}. First we provide several lemmas that help us establish all assumptions.

\subsection{Weights Stay Close to Initial}

To show that weights do not move far from their initialization \eqref{eq:abstract:weight-diff} we use the following lemma.

\begin{lemma} \label{lemma:weight-distance}
  The gradient descent iterates $\theta^n$ with learning rate $\lr$ of the network \eqref{eq:setup:network} with $L_2(\dom)$ loss \eqref{eq:setup:loss} satisfy
  \begin{equation*}
    \|\theta^n - \theta^0\|_\infty \le \frac{2\lr}{\sqrt{m}} \sum_{k=0}^{n-1} \|\res^k\|_{L_2(\dom)}.
  \end{equation*}
\end{lemma}

\begin{proof}
 
We estimate each component $\theta_r$ of $\theta$ by the telescopic sum
\begin{multline*}
  |\theta_r^n - \theta_r^0|
  \le \sum_{k=0}^{n-1} |\theta_r^{k+1} - \theta_r^k|
  \le \sum_{k=0}^{n-1} |\lr \partial_r \loss(\theta^k)|
  \\
  \le \frac{\lr}{\sqrt{m}} \sum_{k=0}^{n-1} |\dualp{\res^k, a_r \dactivation(\cdot - b_r)}|
  \le \frac{2\lr}{\sqrt{m}} \sum_{k=0}^{n-1} \|\res^k\|_{L_2(\dom)},
\end{multline*}
where we have used that $a_r = \pm 1$ and $\|a_r \dactivation(\cdot - b_r)\|_{L_2(\dom)} \le 2$.

\end{proof}

\subsection{Results from \texorpdfstring{\cite{GentileWelper2022a}}{}}

This section summarizes some lemmas from \cite{GentileWelper2022a}, which proves gradient flow instead of gradient descent convergence. These will be used to establish assumptions of Theorem \ref{th:general:convergence}.

\begin{lemma}[{\cite[Lemma 5.5]{GentileWelper2022a}}] \label{lemma:ao:shallow:derivative-bounds}

  For the shallow network \eqref{eq:setup:network} and $\sm < \frac{1}{2}$, the partial derivatives $\partial_r f_\theta$ depend only on $\theta_r$ and we have $\|\partial_r f_\theta\|_\sm \le \frac{\mu}{\sqrt{m}}$ for some $\mu > 0$ independent of $m$.

\end{lemma}

\begin{lemma}[{\cite[Lemma 5.7]{GentileWelper2022a}}] \label{lemma:ao:shallow:perturbation-local}

  For the shallow network \eqref{eq:setup:network}, let the weights $\theta \in \wdom$, be i.i.d. uniformly distributed on $\wdom$ and assume that $0 \le \sm < \frac{1}{2}$. Then for any $h \ge 0$, with probability at least $1 - \frac{c}{h} e^{-2m \wdiff^2}$, we have
  \[
    \sup_{\|\nu\|_\infty \le 1} \left\| \sum_{r=1}^m (\partial_r f_\theta - \partial_r f_\ptheta) \nu_r \right\|_\sm \le c \sqrt{m} \wdiff^{1-\sm}
  \]
  for all $\ptheta \in \wdom$ with $\|\theta - \ptheta\|_\infty \le \wdiff$ and some constant $c > 0$ independent of $m$.

\end{lemma}

For the following results, we use the induced operator norm $\|H\|_{S,0}$ for $H: \hs^0 \to \hs^S$. Note that in the cited papers use the notation $\|H\|_{0,S}$, instead.

\begin{lemma}[{Analogous to \cite[Lemma 4.3]{GentileWelper2022a}}] \label{lemma:ao:perturbation}
  Assume there are constants $\alpha, \mu, L \ge 0$ so that for $S=0$ and $S=\sm$ and all $\theta^0, \ptheta, \pptheta \in \wdom$ with $\|\theta - \ptheta\|_\infty \lesssim h$ we have
  \begin{align*}
  \|\partial_r f_{\pptheta}\|_S & \le \frac{\mu}{\sqrt{m}}, &
  \sup_{\|\nu\|_\infty \le 1} \left\| \sum_{r=1}^m (\partial_r f_{\theta^0} - \partial_r f_\ptheta) \nu_r \right\|_S & \le \sqrt{m} L h^\alpha
  \end{align*}
  Then
  \begin{align*}
    \|H_{\pptheta, \theta^0}  - H_{\pptheta, \ptheta}\|_{S,0} & \le \mu L h^\alpha, & 
    \|H_{\theta^0, \pptheta}  - H_{\ptheta, \pptheta}\|_{S,0} & \le \mu L h^\alpha.
  \end{align*}
\end{lemma}

\begin{proof}
  The proof of the first inequality $\|H_{\pptheta, \theta^0}  - H_{\pptheta, \ptheta}\|_{S,0} \le 2 \mu L h^\alpha$, is identical to the bounds for $S_1$ in the proof of \cite[Lemma 4.3]{GentileWelper2022a}, with the only difference that in the latter $\ptheta = \pptheta$. Likewise, the bounds for the second inequality $\|H_{\theta^0, \pptheta}  - H_{\ptheta, \pptheta}\|_{S,0} \le 2 \mu L h^\alpha$ is identical to $S_2$ in the reference.

\end{proof}

\begin{lemma}[{\cite[Lemma 4.2]{GentileWelper2022a}}] \label{lemma:ao:initial-concentration}

  Assume that for the shallow network \eqref{eq:setup:network} the partial derivatives $\partial_r f_\theta$, $r=1, \dots, m$ only depend on the single weight $\theta_r$ and that $\|\partial_r f_\theta\|_S \le \frac{\mu}{\sqrt{m}}$ for $S \in \{0, \sm\}$, $\sm \in \real$. Then for independently sampled initial weights $\theta_r$ and all $\tau > 0$, we have
  \[
    \pr{ \|H_{\theta,\theta} - H\|_{S,0} \ge \sqrt{\frac{8 \mu^4 \tau}{m}} + \frac{2 \mu^2 \tau}{3m}}
    \le 2 \tau \left(e^\tau - \tau - 1 \right)^{-1}.
  \]

\end{lemma}

\subsection{Proof of Main Result}
\label{sec:1d:convergence}

\begin{proof}[Proof of Theorem \ref{th:1d:convergence}]

The result follows from Theorem \ref{th:general:convergence}, with assumptions satisfied as follows.
\begin{enumerate}

  \item Coercivity with $\ntks=1$ is shown in \cite[Section 5.5, Proof of Theorem 5.1, Item 3]{GentileWelper2022a}.

  \item With $\wnorm{\cdot} = \|\cdot\|_\infty$, by Lemma \ref{lemma:weight-distance} we have
  \begin{equation*}
    \|\theta^n - \theta^0\|_\infty \le \frac{2\lr}{\sqrt{m}} \sum_{k=0}^{n-1} \|\res^k\|_{L_2(\dom)}
  \end{equation*}
  so that \eqref{eq:abstract:weight-diff} is satisfied.

\item Since $\|\partial_r f_\theta\|_S \le \frac{\mu}{\sqrt{m}}$ by Lemma \ref{lemma:ao:shallow:derivative-bounds} and $\lr \lesssim \mu^{-1} h \sqrt{m}$ by assumption, we obtain \eqref{eq:abstract:lr}
  \begin{equation*}
    \gamma \|\partial_r f_\theta\|_S 
    \le \left( \mu^{-1} h^\alpha \sqrt{m} \right) \left(\frac{\mu}{\sqrt{m}}\right)
    = h^\alpha.
  \end{equation*}

  \item By Lemmas \ref{lemma:ao:shallow:derivative-bounds} and \ref{lemma:ao:shallow:perturbation-local}, with probability at least $1 - \frac{c}{h} e^{-2m \wdiff^2}$ all assumptions of Lemma \ref{lemma:ao:perturbation} are satisfied, which directly implies the perturbation assumption \eqref{eq:assume-ntk-perturbations}, with $\alpha = 1-\sm$.

  \item Our choice of $\tau = h^{2\alpha} m$ implies $\sqrt{\frac{\tau}{m}} \le h^\alpha \lesssim 1$ so that from Lemmas \ref{lemma:ao:shallow:derivative-bounds}, \ref{lemma:ao:initial-concentration} with probability at least $1 - 2 \tau \left(e^\tau - \tau - 1 \right)^{-1}$ we have
  \[
    \|H_{\theta,\theta} - H\|_{S,0} \le \sqrt{\frac{8 \mu^4 \tau}{m}} + \frac{2 \mu^2 \tau}{3m} \lesssim h^\alpha,
  \]
  which shows the initial concentration \eqref{eq:assume-initial-close-to-ntk}.
\end{enumerate}
The random events (w.r.t. initialization) in the last two items are satisfied, with high probability by a union bound. In this case, all assumptions of Theorem \ref{th:general:convergence} are true and the result follows, with $\alpha = 1-\sm$.

\end{proof}

\section{Proof of Main Results: Deep \texorpdfstring{$nd$}{nd}}
\label{sec:proof-deep}

In this section, we proof Theorem \ref{th:deep:convergence} as a special case of Theorem \ref{th:general:convergence}. First we provide several lemmas that help us establish all assumptions.

\subsection{Setup}

We denote the weights on layer $\ell$ at gradient descent step $n$ by $W^\ell(n)$ and we repeatedly use the properties
\begin{equation} \label{eq:deep:assumption:activation-growth}
  |\activationp{x}| \lesssim |x|,
\end{equation}
\begin{equation} \label{eq:deep:assumption:activation-lipschitz}
  |\activationp{x} - \activationp{\px}| \lesssim |x - \px|
\end{equation}
\begin{equation} \label{eq:deep:assumption:dactivation-bounded}
  |\dactivationp{x}| \lesssim 1.
\end{equation}
of the activation functions. Instead of $H_{\ptheta, \pptheta}$ for the shallow case, we use the corresponding integral kernels, or empirical NTKs
\begin{equation*}
  \entk(x,y) := \sum_{|\w| = L-1} \partial_\w f_r^{L+1}(x) \partial_\w f_r^{L+1}(y).
\end{equation*}
and define the NTK by the infinite width limit (with random weights at initialization)
\begin{equation*}
  \ntk(x,y) := \lim_{width \to \infty} \sum_{|\w| = L-1} \partial_\w f_r^{L+1}(x) \partial_\w f_r^{L+1}(y).
\end{equation*}
The induced integral operators $H\res = \int_\dom \ntk(\cdot, y) \res(y) \, dy$ and $\hat{H} \res$ correspond to the operators used for the shallow case. Unlike $H$ and $\hat{H}$, we analyze $\ntk$ and $\entk$ in Hölder-norms $\|\cdot\|_\CHHN{\sm}{\tm}$ with $\sm$ and $\tm$ Hölder continuity in $x$ and $y$, respectively. See \cite[Section 6.1]{Welper2023a} for rigorous definitions.

\subsection{Weights Stay Close to Initial}

To show that weights do not move far from their initialization \eqref{eq:abstract:weight-diff} we use the following results. To this end, let $\|v\|_{C^0(\dom)}$ and $\|W\|_{C^0(\dom)}$ be the maximum norm of vector and matrix valued functions $v(x)$ and $W(x)$ with Euclidean and spectral norm for the respective image spaces.

\begin{lemma}[{\cite[Lemma 5.18, special case for last layer $\ell=L+1$]{Welper2023a}}] \label{lemma:deep:partial-bound}
  Assume that $\activation$ satisfies the growth and derivative bounds \eqref{eq:deep:assumption:activation-growth}, \eqref{eq:deep:assumption:dactivation-bounded} and may be different in each layer. Assume the weights are bounded $\|W^\ell\| m_\ell^{-1/2} \lesssim 1$, $\ell=1, \dots, L$. Then
  \begin{equation*}
    \left\| \partial_{W^\ell} f^{L+1} \right\|_\CND{0}{\dom} 
    \lesssim \left(\frac{m_0}{m_\ell}\right)^{1/2}.
  \end{equation*}
  
\end{lemma}

\begin{lemma} \label{lemma:deep:close-to-initial}
  
  Assume that $\activation$ satisfies the growth and derivative bounds \eqref{eq:deep:assumption:activation-growth}, \eqref{eq:deep:assumption:dactivation-bounded} and may be different in each layer. Assume the weights are defined by gradient descent \eqref{eq:setup:gd} and satisfy
  \begin{align*}
    \|W^\ell(0)\| m_\ell^{-1/2} & \lesssim 1, &
    \ell & =1, \dots, L, &
    \\
    \|W^\ell(k) - W^\ell(0)\| m_\ell^{-1/2} & \lesssim 1, &
    0 & \le k < n.
  \end{align*}
  Then
  \begin{equation*}
    \left\|W^\ell(n) - W^\ell(0)\right\| m_\ell^{-1/2}
    \lesssim \lr \frac{m_0^{1/2}}{m_\ell} \sum_{k=0}^{n-1} \|\res^k\|_{\CND{0}{\dom}'},
  \end{equation*}
  where $\CND{0}{\dom}'$ is the dual space of $\CND{0}{\dom}$.

\end{lemma}

\begin{proof}

  The proof is analogous to \cite[Lemma 7.2]{Welper2023a} for gradient flow instead of gradient descent. By assumption, we have
\begin{align*}
  \|W^\ell(k)\| m_\ell^{-1/2} & \lesssim 1, &
  0 & \le k < n, &
  \ell & =1, \dots, L.
\end{align*}
With loss $\loss$, residual $\res^k = f_{\theta^k}-f$, gradient descent step
\begin{equation*}
  W^\ell(k+1) - W^\ell(k) 
  = - \lr \nabla_{W^\ell} \loss
  = - \lr \int_\dom \res^k(x) \partial_{W_\ell} f^{L+1}(x) \, dx
\end{equation*}
and a telescopic sum, we have
\begin{align*}
  \left\|W^\ell(n) - W^\ell(0)\right\|
  & = \left\|\sum_{k=0}^{n-1} W^\ell(k+1) - W^\ell(k) \right\|
  \\
  & = \lr \left\|\sum_{k=0}^{n-1} \int_\dom \res^k(x) \partial_{W_\ell} f^{L+1}(x) \, dx \right\|
  \\
  & \le \lr \sum_{k=0}^{n-1} \int_\dom |\res^k(x)| \left\|\partial_{W_\ell} f^{L+1}(x) \right\| \, dx
  \\
  & \lesssim \lr \left(\frac{m_0}{m_\ell}\right)^{1/2} \sum_{k=0}^{n-1} \|\res^k\|_{\CND{0}{\dom}'},
\end{align*}
where in the last step we have used Lemma \ref{lemma:deep:partial-bound}. Multiplying with $m_\ell^{-1/2}$ shows the result.

\end{proof}

\subsection{Gradient Bounds}

We bound the gradients as required for assumption \eqref{eq:abstract:lr} in Theorem \ref{th:general:convergence}.

\begin{lemma} \label{lemma:deep:gradient-bound}
  
  Assume that $\activation$ satisfies the growth and Lipschitz conditions \eqref{eq:deep:assumption:activation-growth}, \eqref{eq:deep:assumption:activation-lipschitz} and may be different in each layer. Assume the weights $\|W^\ell(n)\| m_\ell^{-1/2} \lesssim 1$ are bounded and $m_L \sim m_{L-1} \sim \dots \sim m_0$. Then
  \begin{equation*}
    \|\nabla_\theta \loss(\theta^n)\|
    \lesssim \|\res^n\|_{L_2(\dom)}.
  \end{equation*}

\end{lemma}

\begin{proof}

Choosing $\pW^{L-1} = W^{L-1}(n)$ as the gradient descent iterate, an elementary computation shows that
\begin{equation*}
  \partial_{\pW_{ij}^{L-1}} \pf_r^{L+1}
  = \pW_i^L m_L^{-1/2} m_{L-1}^{-1/2} \dactivationp{\pf_i^L} \activationp{\pf_j^{L-1}},
\end{equation*}
where the last weight $\pW^L$ is a vector because the network is scalar valued, see e.g. \cite[Proof of Lemma 4.1]{Welper2023a}. Since we only optimize layer $L-1$, it follows that the gradient
\begin{equation*}
  \nabla_\theta \loss(\theta^n)
  = m_L^{-1/2} m_{L-1}^{-1/2} \activationp{\pf^{L-1}} \dualp{\res^n, \pW^L \odot \dactivationp{\pf^L}}
  =: m_L^{-1/2} m_{L-1}^{-1/2} u v^T,
\end{equation*}
with element-wise product $\odot$ is a rank $1$ matrix with spectral norm $\|uv^T\| = \|u\|\|v\|$. From \cite[Lemma 5.5]{Welper2023a} applied to $\activation$ and $\dactivation$, we have
\begin{align*}
  \left\|\activationp{\pf^\ell}\right\|_\CN{0}
  & \lesssim m_0^{1/2}, &
  \left\|\dactivationp{\pf^\ell}\right\|_\CN{0}
  & \lesssim m_0^{1/2}, &
  \ell & = 1, \dots, L+1.
\end{align*}
Thus, with $W_{\cdot, i}^L = \pm 1$ and $m_L \sim m_{L-1} \sim \dots \sim m_0$, we conclude that
\begin{equation*}
  \|\nabla_\theta \loss(\theta^n)\|
  \lesssim \|\res^n\|_{L_2(\dom)}.
\end{equation*}
which shows the lemma.

\end{proof}

\subsection{Perturbations}

In this section, we show the perturbation assumption \eqref{eq:assume-ntk-perturbations} of Theorem \ref{th:general:convergence}. We denote two separate perturbations with an extra $\bar{\cdot}$ and $\tilde{\cdot}$, so that we have weights $\pW^\ell$, $\ppW^\ell$ with respective network evaluations $\pf^\ell(x)$, $\ppf^\ell(x)$ as well as the perturbed empirical NTKs
\begin{equation*}
  \pentk(x,y) := \sum_{|\w| = L-1} \partial_\w f_r^{L+1}(x) \partial_\w \pf_r^{L+1}(y).
\end{equation*}
and 
\begin{equation*}
  \ppentk(x,y) := \sum_{|\w| = L-1} \partial_\w \ppf_r^{L+1}(x) \partial_\w \pf_r^{L+1}(y).
\end{equation*}

The initial random weight matrices $W^\ell := W^\ell(0)$ are bounded with high probability and because weigts do not move far from their initial by Lemma \ref{lemma:deep:close-to-initial}, all relevant perturbations will have the same property. Therefore, for now we assume that
\begin{align} \label{eq:deep:assumption:weights-domain-bounded}
   \left\|W^\ell\right\| m_\ell^{-1/2} & \lesssim 1, & 
   \left\|\pW^\ell\right\| m_\ell^{-1/2} & \lesssim 1, &
   \left\|\ppW^\ell\right\| m_\ell^{-1/2} & \lesssim 1, &
   \|x\| & \lesssim 1 \, \forall x \in \dom.
\end{align}

\begin{lemma}[{\cite[Lemma 4.3]{Welper2023a}}] \label{lemma:deep:continuity:entk-holder}
  
  Assume that $\activation$ and $\dactivation$ satisfy the growth and Lipschitz conditions \eqref{eq:deep:assumption:activation-growth}, \eqref{eq:deep:assumption:activation-lipschitz} and may be different in each layer. Assume the weights, perturbed weights and domain are bounded \eqref{eq:deep:assumption:weights-domain-bounded} and $m_L \sim m_{L-1} \sim \dots \sim m_0$. Then for $0 < \sm < 1$
  \begin{align*}
    \left\|\pentk - \ppentk\right\|_\CHHN{\sm}{\sm}
                                                  & \lesssim \frac{m_0}{m_L} \left[ \sum_{k=0}^{L-1} \max_{V^k=\pW^k, \ppW^k}\left\|W^k - V^k \right\| m_k^{-1/2} \right]^{1-\sm}.
  \end{align*}

\end{lemma}

\begin{proof}
  The reference \cite[Lemma 4.3]{Welper2023a} only considers the case that the two perturbations $\pW^\ell = \ppW^\ell$ are identical. Howver, the proof remains unchanged, except that we have to maximize over both perturbation in the right hand side. This originates form the proof of the intermediate \cite[Lemma 5.6]{Welper2023a}.
\end{proof}

\subsection{Concentration}

The concentration inequality \eqref{eq:assume-initial-close-to-ntk} of Theorem \ref{th:general:convergence} is provided by the following lemma.

\begin{lemma}[{\cite[Lemma 4.4]{Welper2023a}}] \label{lemma:deep:concentration:concentration-ntk}
  Let $\sm = \tm = 1/2$ and $k = 0, \dots, L-1$.
  \begin{enumerate}
    \item Assume that $W^L \in \{-1, +1\}$ with probability $1/2$ each.
    \item Assume that all $W^k$ are are i.i.d. standard normal. 
    \item Assume that $\activation$ and $\dactivation$ satisfy the growth condition \eqref{eq:deep:assumption:activation-growth}, have uniformly bounded derivatives \eqref{eq:deep:assumption:dactivation-bounded}, derivatives $\activation^{(i)}$, $i=0, \dots, 3$ are continuous and have at most polynomial growth for $x \to \pm \infty$ and the scaled activations satisfy
    \begin{align*}
      \left\|\partial^i (\activation_a) \right\|_N & \lesssim 1, & 
      \left\|\partial^i (\dactivation_a) \right\|_N & \lesssim 1, & 
      a & \in \{\gp^k(x,x): x \in \dom\}, &
      i & = 1, \dots, 3,
    \end{align*}
    with $\activation_a(x) := \activation(ax)$ and
    \begin{align*}
      \|f\|_N^2 := \int_\real f(x)^2 d\gaussian{0,1}(x).
    \end{align*}
    The activation functions may be different in each layer.
    \item For all $x \in \dom$ assume
    \begin{equation*}
      \gp^k(x,x) \ge c_\gp > 0.
    \end{equation*}
    \item The widths satisfy $m_\ell \gtrsim m_0$ for all $\ell=0, \dots, L$.
  \end{enumerate}
  Then, with probability at least
  \begin{equation*}
    1 - c \sum_{k=1}^{L-1} e^{-m_k} + e^{-u_k}
  \end{equation*}
  we have
  \begin{equation*}
    \left\| \entk - \ntk \right\|_\CHHN{\sm}{\tm} 
    \lesssim \sum_{k=0}^{L-1} \frac{m_0}{m_k}\left[ \frac{\sqrt{d}+\sqrt{u_k}}{\sqrt{m_k}} + \frac{d+u_k}{m_k} \right]
    \le \frac{1}{2} c_\gp
  \end{equation*}
  for all $u_1, \dots, u_{L-1} \ge 0$ sufficiently small so that the rightmost inequality holds.

\end{lemma}

\subsection{Bounds for Integral Kenrels}

The above lemmas provide perturbation and concentration results for the kernels $\ntk$ and $\entk$ in Hölder-norms $\|\cdot\|_\CHHN{\sm}{\sm}$. These imply bounds for the corresponding integral operators $H$ and $\hat{H}$ by the following lemma.

\begin{lemma}[{\cite[Lemma 6.16]{Welper2023a}}] \label{lemma:deep:supplements:kernel-bound}

  Let $0 < \sm, \tm < 1$. Then for any $\epsilon > 0$ with $\sm+\epsilon \le 1$ and $\tm+\epsilon<1$, we have
  \begin{equation*}
    \iint_{\dom \times \dom} f(x) k(x,y) g(y) \, dx \, dy
    \le \|f\|_{H^{-\sm}(\Sd)} \|g\|_{H^{-\tm}(\Sd)} \|k\|_\CHHND{\sm+\epsilon}{\tm+\epsilon}{\Sd}.
  \end{equation*}
  
\end{lemma}

\subsection{Proof of Main Result}
\label{sec:deep:convergence}

\begin{proof}[Proof of Theorem \ref{th:deep:convergence}]

  We prove the result with Theorem \ref{th:general:convergence}. To establish its assumptions with the preceding lemmas, we need to bound the weights. To this end, we define
  \begin{equation*}
    \wnorm{\cdot} := \max_{0 \le \ell \le L} \|\cdot\| m_\ell^{1/2}
  \end{equation*}
  with spectral norm $\|\cdot\|$. The initial weights satisfy $\|W(0)^\ell\| m_\ell^{-1/2} \lesssim 1$, with probability at least $1 - 2 e^{-c m}$ since $m_\ell \sim m$ by assumption, see e.g. \cite[Theorem 4.4.5]{Vershynin2018}. By the conditions on $\wnorm{\theta^0 - \Box}$, $\Box \in \{\theta^{n-1}, \ptheta, \pptheta\}$ in \eqref{eq:abstract:weight-diff} and \eqref{eq:assume-ntk-perturbations} this bound can be extended to gradient descent iterates and perturbations, so that we obtain
\begin{align} \label{eq:deep:proof:1:th:convergence}
  \max\{
    \|W^\ell\|,
    \|\pW^\ell\|,
    \|\ppW^\ell\|
  \} m_\ell^{-1/2} & \lesssim 1, &
  \|W(k)^\ell - W(0)^\ell\| m_\ell^{-1/2} & \lesssim 1, & 
\end{align}
for $\ell = 0, \dots, L$. Now, the result follows from Theorem \ref{th:general:convergence} for which we verify all assumptions.

\begin{enumerate}

  \item Coercivity \eqref{eq:abstract:coercive} is given by assumption \eqref{eq:deep:coercivity}.

  \item The weight distance \eqref{eq:abstract:weight-diff} follows directly from \eqref{eq:deep:proof:1:th:convergence} and Lemma \ref{lemma:deep:close-to-initial}.

  \item Since $\|\res^{n-1}\|_{L_2(\dom)} \le \|\res^{n-1}\|_{H^s(\dom)}$ is uniformly bounded during the gradient descent iteration (by inductive application of Theorem \ref{th:general:convergence}), the gradient satisfies the bound
  \begin{equation*}
    \wnorm{\nabla_\theta \loss(\theta^{n-1})}
    \lesssim m^{-1/2}
  \end{equation*}
by \eqref{eq:deep:proof:1:th:convergence} and Lemma \ref{lemma:deep:gradient-bound}. Hence \eqref{eq:abstract:lr} is satisfied with assumption $\lr \le h\sqrt{m}$.

\item From \eqref{eq:deep:proof:1:th:convergence} and Lemma \ref{lemma:deep:continuity:entk-holder}, for sufficiently small $\epsilon$, we have
  \begin{equation*}
    \left\|\pentk - \ppentk\right\|_\CHHN{\sm+\epsilon}{\sm+\epsilon}
     \lesssim h^{1-(\sm+\epsilon)}
     \lesssim h^\alpha
  \end{equation*}
  for $\alpha := 1-(\sm + \epsilon)$, with some constants that depend on $L$ and $\alpha$. With perturbations $H_{\pptheta, \ptheta^0} \res = \int_\dom \pentk(\cdot, y) \res(y) \, dy$ and $H_{\pptheta, \ptheta} \res = \int_\dom \ppentk(\cdot, y) \res(y) \, dy$ and Lemma \ref{lemma:deep:supplements:kernel-bound}, we obtain
  \begin{equation*}
    \|H_{\pptheta, \theta^0}  - H_{\pptheta, \ptheta}\|_{S,0}
    \le c h^\alpha.
  \end{equation*}
  The bounds for $\|H_{\theta^0, \pptheta}  - H_{\ptheta, \pptheta}\|_{S,0}$ follow analogously and thus \eqref{eq:assume-ntk-perturbations} is satisfied.

  \item With \eqref{eq:deep:gp-bounds} and the given assumptions on the network, we can apply Lemma \ref{lemma:deep:concentration:concentration-ntk} with $u_k = \tau$ and $\sm + \epsilon = 1/2$. Thus, with probability at least
  \begin{equation*}
    1 - c e^{-m_0} + e^{-\tau}
  \end{equation*}
  we have
  \begin{equation*}
  \left\| \entk - \ntk \right\|_\CHHN{\sm+\epsilon}{\sm+\epsilon} 
    \lesssim \sqrt{\frac{d}{m_0}} + \sqrt{\frac{\tau}{m_0}} + \frac{\tau}{m_0}
    \lesssim \sqrt{\frac{\tau}{m_0}}
    \lesssim h^\alpha,
  \end{equation*}
  where in the last step we have used the definition of $\tau$. With Lemma \ref{lemma:deep:supplements:kernel-bound}, this directly implies
  \begin{equation*}
    \|H - H_{\theta^0}\|_{S,0} \lesssim h^\alpha.
  \end{equation*}
  and therefore \eqref{eq:assume-initial-close-to-ntk}.

\end{enumerate}

Hence all assumptions of Theorem \ref{th:general:convergence} are satisfied with $\alpha < 1-s$ and the result follows.

\end{proof}

\bibliographystyle{abbrv}
\bibliography{aogd}

\begin{thebibliography}{10}

\bibitem{AdcockDexter2020}
B.~Adcock and N.~Dexter.
\newblock The gap between theory and practice in function approximation with
  deep neural networks.
\newblock {\em SIAM Journal on Mathematics of Data Science}, 3(2):624–655,
  2021.

\bibitem{Allen-ZhuLiSong2019}
Z.~Allen-Zhu, Y.~Li, and Z.~Song.
\newblock A convergence theory for deep learning via over-parameterization.
\newblock In K.~Chaudhuri and R.~Salakhutdinov, editors, {\em Proceedings of
  the 36th International Conference on Machine Learning}, volume~97 of {\em
  Proceedings of Machine Learning Research}, page 242–252, Long Beach,
  California, USA, 09–15 Jun 2019. PMLR.
\newblock Full version available at \url{https://arxiv.org/abs/1811.03962}.

\bibitem{AroraDuHuEtAl2019}
S.~Arora, S.~Du, W.~Hu, Z.~Li, and R.~Wang.
\newblock Fine-grained analysis of optimization and generalization for
  overparameterized two-layer neural networks.
\newblock In K.~Chaudhuri and R.~Salakhutdinov, editors, {\em Proceedings of
  the 36th International Conference on Machine Learning}, volume~97 of {\em
  Proceedings of Machine Learning Research}, page 322–332, Long Beach,
  California, USA, 09–15 Jun 2019. PMLR.

\bibitem{AroraDuHuEtAl2019a}
S.~Arora, S.~S. Du, W.~Hu, Z.~Li, R.~R. Salakhutdinov, and R.~Wang.
\newblock On exact computation with an infinitely wide neural net.
\newblock In H.~Wallach, H.~Larochelle, A.~Beygelzimer, F.~d'Alché Buc,
  E.~Fox, and R.~Garnett, editors, {\em Advances in Neural Information
  Processing Systems}, volume~32. Curran Associates, Inc., 2019.

\bibitem{Bach2017}
F.~Bach.
\newblock Breaking the curse of dimensionality with convex neural networks.
\newblock {\em Journal of Machine Learning Research}, 18(19):1–53, 2017.

\bibitem{BaiLee2020}
Y.~Bai and J.~D. Lee.
\newblock Beyond linearization: On quadratic and higher-order approximation of
  wide neural networks.
\newblock In {\em International Conference on Learning Representations}, 2020.

\bibitem{BernerGrohsKutyniokPetersen2021}
J.~Berner, P.~Grohs, G.~Kutyniok, and P.~Petersen.
\newblock The {Modern} {Mathematics} of {Deep} {Learning}.
\newblock In P.~Grohs and G.~Kutyniok, editors, {\em Mathematical {Aspects} of
  {Deep} {Learning}}, page 1–111. Cambridge University Press, 1 edition, Dec.
  2022.

\bibitem{BiettiMairal2019}
A.~Bietti and J.~Mairal.
\newblock On the inductive bias of neural tangent kernels.
\newblock In H.~Wallach, H.~Larochelle, A.~Beygelzimer, F.~d\textquotesingle
  Alché-Buc, E.~Fox, and R.~Garnett, editors, {\em Advances in Neural
  Information Processing Systems}, volume~32. Curran Associates, Inc., 2019.

\bibitem{BreslerNagaraj2020}
G.~Bresler and D.~Nagaraj.
\newblock Sharp representation theorems for {ReLU} networks with precise
  dependence on depth.
\newblock In H.~Larochelle, M.~Ranzato, R.~Hadsell, M.~Balcan, and H.~Lin,
  editors, {\em Advances in Neural Information Processing Systems}, volume~33,
  page 10697–10706. Curran Associates, Inc., 2020.

\bibitem{ChenXu2021}
L.~Chen and S.~Xu.
\newblock Deep neural tangent kernel and laplace kernel have the same rkhs.
\newblock In {\em International Conference on Learning Representations}, 2021.

\bibitem{ChenCaoZouGu2021}
Z.~Chen, Y.~Cao, D.~Zou, and Q.~Gu.
\newblock How much over-parameterization is sufficient to learn deep
  re{\{}lu{\}} networks?
\newblock In {\em International Conference on Learning Representations}, 2021.

\bibitem{ChizatOyallonBach2019}
L.~Chizat, E.~Oyallon, and F.~Bach.
\newblock On lazy training in differentiable programming.
\newblock In H.~Wallach, H.~Larochelle, A.~Beygelzimer, F.~d'Alché Buc,
  E.~Fox, and R.~Garnett, editors, {\em Advances in Neural Information
  Processing Systems}, volume~32. Curran Associates, Inc., 2019.

\bibitem{DaubechiesDeVoreFoucartEtAl2019}
I.~Daubechies, R.~DeVore, S.~Foucart, B.~Hanin, and G.~Petrova.
\newblock Nonlinear {Approximation} and ({Deep}) $\mathrm{ReLU}$ {Networks}.
\newblock {\em Constructive Approximation}, 55(1):127–172, Feb. 2022.

\bibitem{DeVoreHaninPetrova2020}
R.~DeVore, B.~Hanin, and G.~Petrova.
\newblock Neural network approximation.
\newblock {\em Acta Numerica}, 30:327–444, 2021.

\bibitem{DrewsKohler2022}
S.~Drews and M.~Kohler.
\newblock On the universal consistency of an over-parametrized deep neural
  network estimate learned by gradient descent, 2022.
\newblock \url{https://arxiv.org/abs/2208.14283}.

\bibitem{DuLeeLiEtAl2019}
S.~Du, J.~Lee, H.~Li, L.~Wang, and X.~Zhai.
\newblock Gradient descent finds global minima of deep neural networks.
\newblock In K.~Chaudhuri and R.~Salakhutdinov, editors, {\em Proceedings of
  the 36th International Conference on Machine Learning}, volume~97 of {\em
  Proceedings of Machine Learning Research}, page 1675–1685, Long Beach,
  California, USA, 09–15 Jun 2019. PMLR.

\bibitem{DuZhaiPoczosSingh2019}
S.~S. Du, X.~Zhai, B.~Poczos, and A.~Singh.
\newblock Gradient descent provably optimizes over-parameterized neural
  networks.
\newblock In {\em International Conference on Learning Representations}, 2019.

\bibitem{GeifmanYadavKastenGalunJacobsRonen2020}
A.~Geifman, A.~Yadav, Y.~Kasten, M.~Galun, D.~Jacobs, and B.~Ronen.
\newblock On the similarity between the laplace and neural tangent kernels.
\newblock In H.~Larochelle, M.~Ranzato, R.~Hadsell, M.~Balcan, and H.~Lin,
  editors, {\em Advances in Neural Information Processing Systems}, volume~33,
  page 1451–1461. Curran Associates, Inc., 2020.

\bibitem{GentileWelper2022a}
R.~Gentile and G.~Welper.
\newblock Approximation results for gradient descent trained shallow neural
  networks in $1d$, 2022.
\newblock \url{https://arxiv.org/abs/2209.08399}.

\bibitem{GribonvalKutyniokNielsenEtAl2019}
R.~Gribonval, G.~Kutyniok, M.~Nielsen, and F.~Voigtlaender.
\newblock Approximation {Spaces} of {Deep} {Neural} {Networks}.
\newblock {\em Constructive Approximation}, 55(1):259–367, Feb. 2022.

\bibitem{GrohsVoigtlaender2021}
P.~Grohs and F.~Voigtlaender.
\newblock Proof of the {Theory}-to-{Practice} {Gap} in {Deep} {Learning} via
  {Sampling} {Complexity} bounds for {Neural} {Network} {Approximation}
  {Spaces}.
\newblock {\em Foundations of Computational Mathematics}, July 2023.

\bibitem{GuhringKutyniokPetersen2020}
I.~Gühring, G.~Kutyniok, and P.~Petersen.
\newblock Error bounds for approximations with deep {ReLU} neural networks in
  ws,p norms.
\newblock {\em Analysis and Applications}, 18(05):803–859, 2020.

\bibitem{HerrmannOpschoorSchwab2022}
L.~Herrmann, J.~A.~A. Opschoor, and C.~Schwab.
\newblock Constructive deep {ReLU} neural network approximation.
\newblock {\em Journal of Scientific Computing}, 90(2):75, 2022.

\bibitem{IbragimovJentzenRiekert2022}
S.~Ibragimov, A.~Jentzen, and A.~Riekert.
\newblock Convergence to good non-optimal critical points in the training of
  neural networks: Gradient descent optimization with one random initialization
  overcomes all bad non-global local minima with high probability, 2022.
\newblock \url{https://arxiv.org/abs/2212.13111}.

\bibitem{JacotGabrielHongler2018}
A.~Jacot, F.~Gabriel, and C.~Hongler.
\newblock Neural tangent kernel: Convergence and generalization in neural
  networks.
\newblock In S.~Bengio, H.~Wallach, H.~Larochelle, K.~Grauman, N.~Cesa-Bianchi,
  and R.~Garnett, editors, {\em Advances in Neural Information Processing
  Systems}, volume~31. Curran Associates, Inc., 2018.

\bibitem{JentzenRiekert2022}
A.~Jentzen and A.~Riekert.
\newblock A proof of convergence for the gradient descent optimization method
  with random initializations in the training of neural networks with relu
  activation for piecewise linear target functions.
\newblock {\em Journal of Machine Learning Research}, 23(260):1–50, 2022.

\bibitem{JiTelgarsky2020}
Z.~Ji and M.~Telgarsky.
\newblock Polylogarithmic width suffices for gradient descent to achieve
  arbitrarily small test error with shallow {ReLU} networks.
\newblock In {\em International Conference on Learning Representations}, 2020.

\bibitem{KawaguchiHuang2019}
K.~Kawaguchi and J.~Huang.
\newblock Gradient descent finds global minima for generalizable deep neural
  networks of practical sizes.
\newblock In {\em 2019 57th Annual Allerton Conference on Communication,
  Control, and Computing (Allerton)}, page 92–99, 2019.

\bibitem{KlusowskiBarron2018}
J.~M. Klusowski and A.~R. Barron.
\newblock Approximation by combinations of {ReLU} and squared {ReLU} ridge
  functions with $\ell{}^1$ and $\ell{}^0$ controls.
\newblock {\em IEEE Transactions on Information Theory}, 64(12):7649–7656,
  2018.

\bibitem{KohlerKrzyzak2022}
M.~Kohler and A.~Krzyzak.
\newblock Analysis of the rate of convergence of an over-parametrized deep
  neural network estimate learned by gradient descent, 2022.
\newblock \url{https://arxiv.org/abs/2210.01443}.

\bibitem{LeeChoiRyuNo2022}
J.~Lee, J.~Y. Choi, E.~K. Ryu, and A.~No.
\newblock Neural tangent kernel analysis of deep narrow neural networks.
\newblock In K.~Chaudhuri, S.~Jegelka, L.~Song, C.~Szepesvari, G.~Niu, and
  S.~Sabato, editors, {\em Proceedings of the 39th International Conference on
  Machine Learning}, volume 162 of {\em Proceedings of Machine Learning
  Research}, page 12282–12351. PMLR, 17–23 Jul 2022.

\bibitem{LeeXiaoSchoenholzBahriNovakSohlDicksteinPennington2019}
J.~Lee, L.~Xiao, S.~Schoenholz, Y.~Bahri, R.~Novak, J.~Sohl-Dickstein, and
  J.~Pennington.
\newblock Wide neural networks of any depth evolve as linear models under
  gradient descent.
\newblock In H.~Wallach, H.~Larochelle, A.~Beygelzimer, F.~d\textquotesingle
  Alché-Buc, E.~Fox, and R.~Garnett, editors, {\em Advances in Neural
  Information Processing Systems}, volume~32. Curran Associates, Inc., 2019.

\bibitem{LiTangYu2019}
B.~Li, S.~Tang, and H.~Yu.
\newblock Better approximations of high dimensional smooth functions by deep
  neural networks with rectified power units.
\newblock {\em Communications in Computational Physics}, 27(2):379–411, 2019.

\bibitem{LiLiang2018}
Y.~Li and Y.~Liang.
\newblock Learning overparameterized neural networks via stochastic gradient
  descent on structured data.
\newblock In S.~Bengio, H.~Wallach, H.~Larochelle, K.~Grauman, N.~Cesa-Bianchi,
  and R.~Garnett, editors, {\em Advances in Neural Information Processing
  Systems 31}, page 8157–8166. Curran Associates, Inc., 2018.

\bibitem{LiMaWu2020}
Z.~Li, C.~Ma, and L.~Wu.
\newblock Complexity measures for neural networks with general activation
  functions using path-based norms, 2020.
\newblock \url{https://arxiv.org/abs/2009.06132}.

\bibitem{LuShenYangZhang2021}
J.~Lu, Z.~Shen, H.~Yang, and S.~Zhang.
\newblock Deep network approximation for smooth functions.
\newblock {\em SIAM Journal on Mathematical Analysis}, 53(5):5465–5506, 2021.

\bibitem{NguyenMondelli2020}
Q.~N. Nguyen and M.~Mondelli.
\newblock Global convergence of deep networks with one wide layer followed by
  pyramidal topology.
\newblock In H.~Larochelle, M.~Ranzato, R.~Hadsell, M.~Balcan, and H.~Lin,
  editors, {\em Advances in Neural Information Processing Systems}, volume~33,
  page 11961–11972. Curran Associates, Inc., 2020.

\bibitem{OpschoorPetersenSchwab2020}
J.~A.~A. Opschoor, P.~C. Petersen, and C.~Schwab.
\newblock Deep {ReLU} networks and high-order finite element methods.
\newblock {\em Analysis and Applications}, 18(05):715–770, 2020.

\bibitem{OymakSoltanolkotabi2020}
S.~{Oymak} and M.~{Soltanolkotabi}.
\newblock Toward moderate overparameterization: Global convergence guarantees
  for training shallow neural networks.
\newblock {\em IEEE Journal on Selected Areas in Information Theory},
  1(1):84–105, 2020.

\bibitem{Pinkus1999}
A.~Pinkus.
\newblock Approximation theory of the mlp model in neural networks.
\newblock {\em Acta Numerica}, 8:143–195, 1999.

\bibitem{ShenYangZhang2019}
Z.~Shen, H.~Yang, and S.~Zhang.
\newblock Nonlinear approximation via compositions.
\newblock {\em Neural Networks}, 119:74–84, 2019.

\bibitem{HaoJinSiegelXu2021}
J.~W. Siegel, Q.~Hong, X.~Jin, W.~Hao, and J.~Xu.
\newblock Greedy training algorithms for neural networks and applications to
  {PDEs}.
\newblock {\em Journal of Computational Physics}, 484:112084, July 2023.

\bibitem{SiegelXu2020}
J.~W. Siegel and J.~Xu.
\newblock Approximation rates for neural networks with general activation
  functions.
\newblock {\em Neural Networks}, 128:313–321, 2020.

\bibitem{SiegelXu2020a}
J.~W. Siegel and J.~Xu.
\newblock High-order approximation rates for shallow neural networks with
  cosine and $\text{ReLU}^k$ activation functions.
\newblock {\em Applied and Computational Harmonic Analysis}, 58:1–26, 2022.

\bibitem{SiegelXu2022}
J.~W. Siegel and J.~Xu.
\newblock Optimal convergence rates for the orthogonal greedy algorithm.
\newblock {\em IEEE Transactions on Information Theory}, 68(5):3354–3361,
  2022.

\bibitem{SongRamezaniKebryaPethickEftekhariCevher2021}
C.~Song, A.~Ramezani-Kebrya, T.~Pethick, A.~Eftekhari, and V.~Cevher.
\newblock Subquadratic overparameterization for shallow neural networks.
\newblock In M.~Ranzato, A.~Beygelzimer, Y.~Dauphin, P.~Liang, and J.~W.
  Vaughan, editors, {\em Advances in Neural Information Processing Systems},
  volume~34, page 11247–11259. Curran Associates, Inc., 2021.

\bibitem{SongYang2019}
Z.~Song and X.~Yang.
\newblock Quadratic suffices for over-parametrization via matrix chernoff
  bound, 2019.
\newblock \url{https://arxiv.org/abs/1906.03593}.

\bibitem{SuYang2019}
L.~Su and P.~Yang.
\newblock On learning over-parameterized neural networks: A functional
  approximation perspective.
\newblock In H.~Wallach, H.~Larochelle, A.~Beygelzimer, F.~d\textquotesingle
  Alché-Buc, E.~Fox, and R.~Garnett, editors, {\em Advances in Neural
  Information Processing Systems}, volume~32. Curran Associates, Inc., 2019.

\bibitem{Suzuki2019}
T.~Suzuki.
\newblock Adaptivity of deep re{LU} network for learning in besov and mixed
  smooth besov spaces: optimal rate and curse of dimensionality.
\newblock In {\em International Conference on Learning Representations}, 2019.

\bibitem{Vershynin2018}
R.~Vershynin.
\newblock {\em High-dimensional probability: an introduction with applications
  in data science}.
\newblock Number~47 in Cambridge series in statistical and probabilistic
  mathematics. Cambridge University Press, Cambridge ; New York, NY, 2018.

\bibitem{WeinanChaoLeiWojtowytsch2020}
E.~Weinan, M.~Chao, W.~Lei, and S.~Wojtowytsch.
\newblock Towards a mathematical understanding of neural network-based machine
  learning: What we know and what we don't.
\newblock {\em CSIAM Transactions on Applied Mathematics}, 1(4):561–615,
  2020.

\bibitem{WeinanMaWu2019}
E.~Weinan, C.~Ma, and L.~Wu.
\newblock The {Barron} {Space} and the {Flow}-{Induced} {Function} {Spaces} for
  {Neural} {Network} {Models}.
\newblock {\em Constructive Approximation}, 55(1):369–406, Feb. 2022.

\bibitem{Welper2023a}
G.~Welper.
\newblock Approximation results for gradient flow trained neural networks,
  2023.
\newblock Accepted for publication in Journal of Machine Learning,
  \url{https://arxiv.org/abs/2309.04860}.

\bibitem{Yarotsky2017}
D.~Yarotsky.
\newblock Error bounds for approximations with deep {ReLU} networks.
\newblock {\em Neural Networks}, 94:103–114, 2017.

\bibitem{Yarotsky2018}
D.~Yarotsky.
\newblock Optimal approximation of continuous functions by very deep {ReLU}
  networks.
\newblock In S.~Bubeck, V.~Perchet, and P.~Rigollet, editors, {\em Proceedings
  of the 31st Conference On Learning Theory}, volume~75 of {\em Proceedings of
  Machine Learning Research}, page 639–649. PMLR, 06–09 Jul 2018.

\bibitem{YarotskyZhevnerchuk2020}
D.~Yarotsky and A.~Zhevnerchuk.
\newblock The phase diagram of approximation rates for deep neural networks.
\newblock In H.~Larochelle, M.~Ranzato, R.~Hadsell, M.~Balcan, and H.~Lin,
  editors, {\em Advances in Neural Information Processing Systems}, volume~33,
  page 13005–13015. Curran Associates, Inc., 2020.

\bibitem{ZouCaoZhouGu2020}
D.~Zou, Y.~Cao, D.~Zhou, and Q.~Gu.
\newblock Gradient descent optimizes over-parameterized deep {ReLU} networks.
\newblock {\em Machine Learning}, 109(3):467 – 492, 2020.

\bibitem{ZouGu2019}
D.~Zou and Q.~Gu.
\newblock An improved analysis of training over-parameterized deep neural
  networks.
\newblock In H.~Wallach, H.~Larochelle, A.~Beygelzimer, F.~d\textquotesingle
  Alché-Buc, E.~Fox, and R.~Garnett, editors, {\em Advances in Neural
  Information Processing Systems}, volume~32. Curran Associates, Inc., 2019.

\end{thebibliography}

\end{document}